\pdfoutput=1
\documentclass{article}
\usepackage{ijcai22}

\usepackage{times}
\usepackage{soul}
\usepackage{url}
\usepackage[hidelinks]{hyperref}
\usepackage[utf8]{inputenc}
\usepackage[small]{caption}
\usepackage{graphicx}
\usepackage{amsmath}
\usepackage{amsthm}
\usepackage{booktabs}
\usepackage{algorithm}
\usepackage{algorithmicx}
\usepackage{algpseudocode}
\usepackage{multirow}
\usepackage{amssymb}
\newtheorem{theorem}{Theorem}

\title{Multi-Level Firing with Spiking DS-ResNet: Enabling Better and Deeper Directly-Trained Spiking Neural Networks}

\author{
Lang Feng$^1$
\and
Qianhui Liu$^{1}$\and
Huajin Tang$^{1,2}$\and
De Ma$^1$\And
Gang Pan$^{1,2,}$\footnote{Corresponding author.}
\affiliations
$^1$College of Computer Science and Technology, Zhejiang University, Hangzhou, China\\
$^2$Zhejiang Lab, Hangzhou, China
\emails
\{langfeng, qianhuiliu, htang, made, gpan\}@zju.edu.cn
}

\begin{document}
\maketitle
\begin{abstract}
Spiking neural networks (SNNs) are bio-inspired neural networks with asynchronous discrete and sparse characteristics, which have increasingly manifested their superiority in low energy consumption. Recent research is devoted to utilizing spatio-temporal information to directly train SNNs by backpropagation. However, the binary and non-differentiable properties of spike activities force directly trained SNNs to suffer from serious gradient vanishing and network degradation, which greatly limits the performance of directly trained SNNs and prevents them from going deeper. In this paper, we propose a multi-level firing (MLF) method based on the existing spatio-temporal back propagation (STBP) method, and spiking dormant-suppressed residual network (spiking DS-ResNet). MLF enables more efficient gradient propagation and the incremental expression ability of the neurons. Spiking DS-ResNet can efficiently perform identity mapping of discrete spikes, as well as provide a more suitable connection for gradient propagation in deep SNNs. With the proposed method, our model achieves superior performances on a non-neuromorphic dataset and two neuromorphic datasets with much fewer trainable parameters and demonstrates the great ability to combat the gradient vanishing and degradation problem in deep SNNs.
\end{abstract}

\section{Introduction}
Spiking neural networks (SNNs) are developed to realize brain-like information processing~\cite{Maass1997}, which use asynchronous binary spike signals to transmit information and have the ability to process information in both spatial domain (SD) and temporal domain (TD). Besides, the sparsity and event-driven properties of SNNs position them as potential candidates for the implementation of low energy consumption on dedicated neuromorphic hardware. As an example, the energy consumed by SNNs to transmit a spike on neuromorphic hardware is only nJ or pJ~\cite{DiehlCook2015}.

In terms of learning algorithms, existing unsupervised learning algorithms~\cite{Qi2018,Liu2020} are difficult to train deep SNNs. Currently, there are two main learning algorithms for deep SNNs training. One is ANN-SNN conversion learning~\cite{Sengupta2019,Yan2021,Hu2021}, which converts the pre-trained ANN model to the SNN model. Conversion learning can achieve deep SNNs training with competitive results, but it has to consume a large number of timesteps to ensure the coding resolution. Moreover, conversion learning cannot utilize the TD information, making it difficult to train neuromorphic datasets. The other is direct supervised learning~\cite{Wu2018,Gu2019,Liu2022,Zheng2021}, which is the approach taken by this paper. Direct supervised learning has great potential to make full use of spatio-temporal information to train the network and can reduce the demand for timesteps. However, to achieve more efficient direct supervised learning for better and deeper directly-trained SNNs, there are still two challenging issues to overcome.

The first is gradient vanishing. Due to non-differentiable spike activities, approximate derivative~\cite{Neftci2019} has to be adopted to make the gradient available, such as rectangle function and Gaussian cumulative distribution function~\cite{Wu2018}. However, it will raise a problem that the limited width of the approximate derivative causes membrane potentials of a multitude of neurons to fall into the saturation area, where the approximate derivative is zero or a tiny value. Furthermore, the sharp features that have larger values in the feature map cannot be further enhanced due to falling into the saturation area to the right of the approximate derivative caused by excessive membrane potential. This greatly limits the performance of deep SNNs, and the neurons located in this saturation area caused by excessive membrane potential are termed to be \emph{dormant} units in this paper. In the above cases, the gradient propagation will be blocked and unstable, therefore resulting in the gradient vanishing and increasing the difficulty of training deep SNNs.

The second is network degradation, which is terribly serious in deep directly-trained SNNs, even if residual structure \cite{He2016} is adopted. Therefore, existing training methods mainly expand SNNs in width to get improved performance, resulting in a large number of trainable parameters. The above non-differentiable spike activity is one of the reasons for network degradation, and the weak spatial expression ability of binary spike signals is another significant factor. For the widely used spiking neuron models like leaky integrate-and-fire (LIF) model, the sharp feature with a larger value and the non-sharp feature with a smaller value will have the same output in the forward process if the corresponding membrane potentials both exceed the firing threshold. As a result, the loss of information caused by discrete spikes will make residual structures hard to perform identity mapping.

We take steps to address these two challenges for enabling better and deeper directly-trained deep SNNs. We first propose the multi-level firing (MLF) method. MLF expands the non-zero area of the rectangular approximate derivatives by allocating the coverage of approximate derivative of each level. In this way, the membrane potentials of neurons are more likely to fall into the area where the derivative is not zero, so as to alleviate gradient vanishing. Besides, with the activation function of neurons in MLF generating spikes with different thresholds when activating the input, the expression ability of the neurons can be improved. Second, we propose spiking dormant-suppressed residual network (spiking DS-ResNet). Spiking DS-ResNet can efficiently perform identity mapping of discrete spikes as well as reduce the probability of dormant unit generation, making it more suitable for gradient propagation. To demonstrate the effectiveness of our work, we perform experiments on a non-neuromorphic dataset (CIFAR10) and neuromorphic datasets (DVS-Gesture, CIFAR10-DVS). Our model achieves state-of-the-art performances on all datasets with much fewer trainable parameters. Experimental analysis indicates that MLF effectively reduces the proportion of dormant units and improves the performances, and MLF with spiking DS-ResNet allows SNNs to go very deep without degradation.
\section{Related Work}
\textbf{Learning algorithm of deep SNNs}. For deep SNNs, there are two main learning algorithms to achieve competitive performance: (1) indirect supervised learning such as ANN-SNN conversion learning; (2) direct supervised learning, the gradient descent-based backpropagation method.

The purpose of ANN-SNN conversion learning is to make the SNNs have the same input-output mapping as the ANNs. Conversion learning avoids the problem of the weak expression ability of binary spike signals by approximating the spike sequence the real-valued output of ReLU, with which the inevitable conversion loss arises. A lot of works focus on reducing the conversion loss~\cite{Han2020,Yan2021} and achieve competitive performances. However, conversion learning ignores the effective TD information and needs a large number of timesteps to ensure accuracy. As a result, it is often limited to non-neuromorphic datasets and has a serious inference latency.

In recent years, direct supervised learning of SNNs has developed rapidly. From spatial back propagation~\cite{Lee2016} to spatial-temporal back propagation~\cite{Wu2018,Gu2019,Fang2020}, people have realized the utilization of spatial and temporal information for training. On this basis,~\cite{Zheng2021} realized the direct training of large-size networks and achieved state-of-the-art performance on the neuromorphic datasets. However, existing methods didn't solve the problem of the limited width of approximate derivative and weak expression ability of binary spike signals, which makes the direct training of deep SNNs inefficient. Gradient vanishing and network degradation seriously restrict directly-trained SNNs from going very deep, which is what we want to overcome.

\textbf{Gradient vanishing or explosion}. Gradient vanishing or explosion is the shared challenge of deep ANNs and deep SNNs. For deep ANNs, there are quite a few successful methods to address this problem. Batch normalization (BN)~\cite{Ioffe2015} reduces internal covariate shift to avoid gradient vanishing or explosion. The residual structure~\cite{He2016} makes the gradient propagate across layers by introducing shortcut connection, which is one of the most widely used basic blocks in deep learning.

For directly-trained deep SNNs, existing research on the gradient vanishing or explosion problem is limited. It is worth noting that the threshold-dependent batch normalization (tdBN) method proposed by~\cite{Zheng2021} can adjust the firing rate and avoid gradient vanishing or explosion to some extent, which is helpful for our further research on gradient vanishing. On this basis, we will combat the gradient vanishing problem in SD caused by the limited width of the approximate derivative.

\textbf{Deep network degradation}. Network degradation will result in a worse performance of deeper networks than that of shallower networks. For deep ANNs, one of the most successful methods to solve degradation problem is residual structure~\cite{He2016}. It introduces a shortcut connection to increase the identity mapping ability of the network and enable the networks to reach hundreds of layers without degradation greatly expanding the depth of the networks.

For directly-trained deep SNNs, there are few efforts on the degradation problem. Even if tdBN has explored the directly trained deep SNNs with residual structure and made SNNs go deeper, the degradation of deep SNNs is still serious. Our work will try to fill this gap in the field of SNNs.
\section{Preliminaries}
In this section, we review the spatio-temporal back propagation (STBP)~\cite{Wu2018} and the iterative LIF model~\cite{Wu2019} to introduce the foundation of our work.

STBP realizes error backpropagation in both TD and SD for the direct training of SNNs. On this basis,~\cite{Wu2019} develops the iterative LIF model into an easy-to-program version and accelerates the direct training of SNNs. Considering the fully connected network, the forward process of the iterative LIF model can be described as
\begin{small}
\begin{align}
\label{eq1}
x_{i}^{t+1,n}&=\sum_{j=1}^{l(n-1)}w_{ij}^{n}o_j^{t+1,n-1}+b_i^n, \\
\label{eq2}
u_i^{t+1,n}&=k_{\tau}u_i^{t,n}(1-o^{t,n}_i)+x_i^{t+1,n}, \\
\label{eq3}
o_i^{t+1,n}&=f(u_i^{t+1,n}-V_{th})=
\begin{cases}
1,& u_i^{t+1,n}\geq V_{th} \\
0,& u_i^{t+1,n}<V_{th}  \\
\end{cases},
\end{align}
\end{small}
where $k_{\tau}$ is a decay factor. $n$ and $l(n-1)$ denote the $n$-th layer and the number of neurons in the $(n-1)$-th layer respectively. $t$ is time index. $u_i^{t,n}$ and $o_i^{t,n}$ are the membrane potential and the output of the $i$-th neuron in the $n$-th layer at time $t$ respectively. $o_i^{t,n}\in(0,1)$ is generated by the activation function $f(\cdot)$, which is the step function. $V_{th}$ is the firing threshold. When the membrane potential exceeds the firing threshold, the neuron will fire a spike and the membrane potential is reset to zero. $w_{ij}^n$ is the synaptic weight from the $j$-th neuron in the $(n-1)$-th layer to the $i$-th neuron in the $n$-th layer and $b_i^n$ is the bias.

\section{Method}
\subsection{The MLF Method}
\subsubsection{The forward process}
\begin{figure}[t]
\centering
\includegraphics[width=0.87\columnwidth]{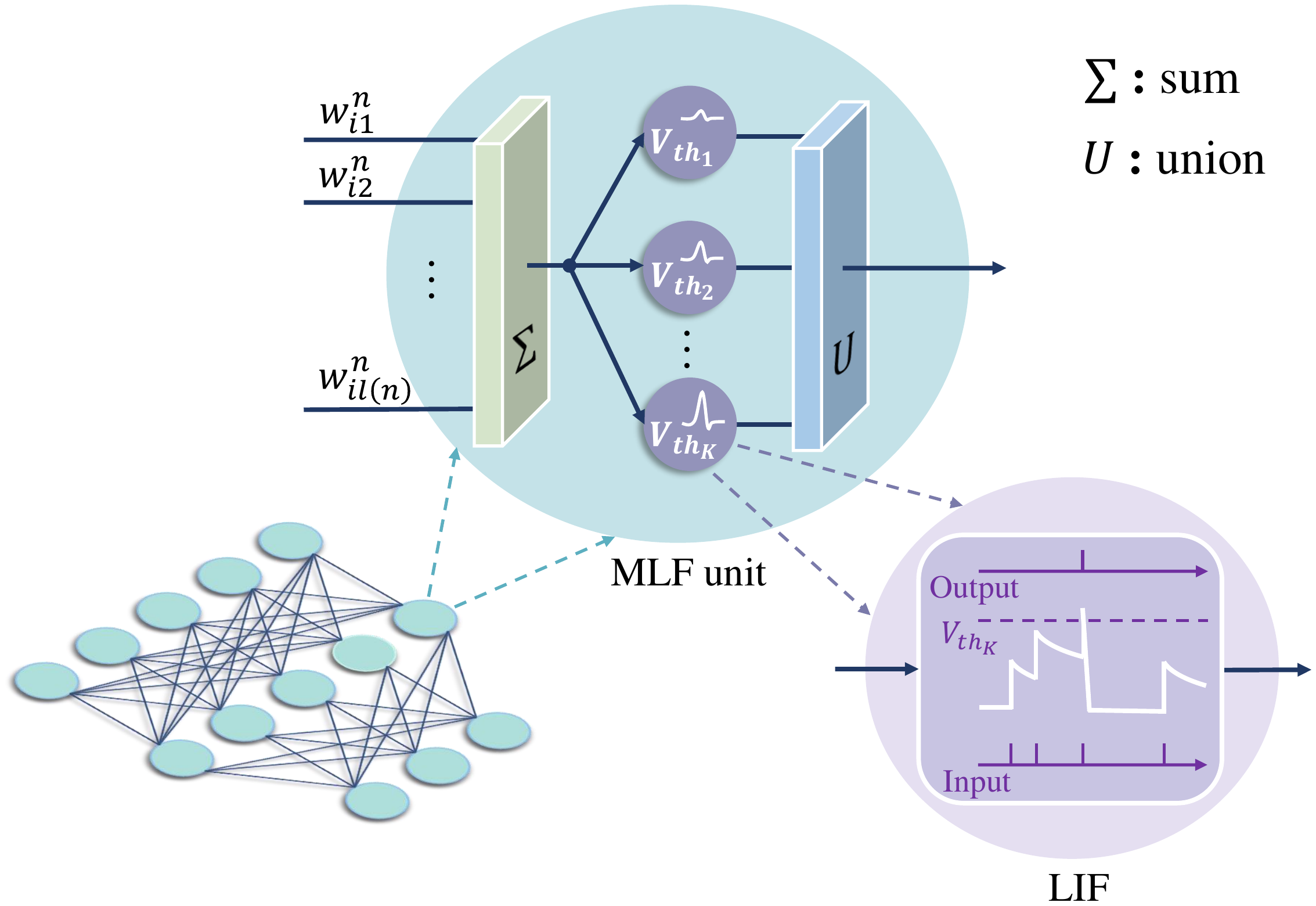} 
\caption{Illustration of MLF unit. A MLF unit contains multiple LIF neurons with different level thresholds. After receiving the input, these neurons will update the membrane potentials. Once the membrane potential of each level neuron reaches the corresponding threshold, a spike will be fired. The final output of MLF unit is the union of the spikes fired by all level neurons.}
\label{fig_MLF}
\end{figure}
As shown in Fig.~\ref{fig_MLF}, we replace LIF neurons with MLF units, which contain multiple LIF neurons with different level thresholds. The output is the union of all spikes fired by these neurons. The forward process can be described as
\begin{align}
\label{eq6}
\textbf{u}_i^{t+1,n}&=k_{\tau} \textbf{u}_i^{t,n}\odot(1-\textbf{o}^{t,n}_i)+x_i^{t+1,n}, \\
\label{eq7}
\textbf{o}_i^{t+1,n}&=f(\textbf{u}_i^{t+1,n}-\textbf{V}_{th}), \\
\label{eq8}
\hat{o}_i^{t+1,n}&=s(\textbf{o}_i^{t+1,n}),
\end{align}
where $\textbf{u}_i^{t,n}=(u_{i,1}^{t,n},u_{i,2}^{t,n},...,u_{i,k}^{t,n},...,u_{i,K}^{t,n})$ and  $\textbf{o}_i^{t,n}=(o_{i,1}^{t,n},o_{i,2}^{t,n},..,o_{i,k}^{t,n},...,o_{i,K}^{t,n})$ denote the membrane potential vector and the output vector of the $i$-th MLF unit in the $n$-th layer at time $t$ respectively. $\odot$ denotes the Hadamard product. $k$ and $K$ denote the $k$-th level and the number of levels respectively. $\textbf{V}_{th}=({V_{th}}_1,{V_{th}}_2,..,{V_{th}}_k,...,{V_{th}}_K)$ is the threshold vector.
To facilitate the calculation of pre-synaptic input $x_i^{t,n}$, we define a spike encoder as $s(\textbf{o}_i^{t,n})=o_{i,1}^{t,n}+o_{i,2}^{t,n}+...+o_{i,K}^{t,n}$, which is completely equivalent to union (see Appendix \ref{appendix_equivalent}). $\hat{o}_i^{t,n}=s(\textbf{o}_i^{t,n})$ is the final output of the $i$-th MLF unit in the $n$-th layer at time $t$. Then, $x_i^{t,n}$ can be computed by Eq.~(\ref{eq1}), where $o_i^{t,n}$ is replaced with $\hat{o}_i^{t,n}$.

Comparing Eq.~(\ref{eq2})-(\ref{eq3}) and Eq.~(\ref{eq6})-(\ref{eq8}), it can be seen that MLF unit doesn't introduce additional trainable parameters to the network, but just replaces LIF neurons with MLF units. Benefitting from the union of multiple spikes, MLF unit can distinguish some sharp features with large values and the non-sharp features with small values.
\subsubsection{The backward process}
To demonstrate that MLF can make the gradient propagation more efficient in SD, we next deduce the backward process of MLF method.

In order to obtain the gradients of weights and biases, we first derive the gradients of $o_{i,k}^{t,n}$, $\hat{o}_{i}^{t,n}$ and $u_{i,k}^{t,n}$, With $L$ representing the loss function, the gradients $\partial L/\partial o_{i,k}^{t,n}$, $\partial L/\partial \hat{o}_{i}^{t,n}$ and $\partial L/\partial u_{i,k}^{t,n}$ can be computed by applying the chain rule as follows
\begin{align}
    \frac{\partial L}{\partial o_{i,k}^{t,n}}&=\frac{\partial L}{\partial \hat{o}_{i}^{t,n}}+\frac{\partial L}{\partial u_{i,k}^{t+1,n}}u_{i,k}^{t,n}(-k_{\tau}), \label{eq12} \\
    \frac{\partial L}{\partial \hat{o}_{i}^{t,n}}&=\sum_{j=1}^{l(n+1)}\sum_{m=1}^K(\frac{\partial L}{\partial u_{j,m}^{t,n+1}}w_{ji}^n), \label{eq12_new}\\
    \frac{\partial L}{\partial u_{i,k}^{t,n}}&=\frac{\partial L}{\partial o_{i,k}^{t,n}}\frac{\partial o_{i,k}^{t,n}}{\partial u_{i,k}^{t,n}}+
    \frac{\partial L}{\partial u_{i,k}^{t+1,n}}k_{\tau}(1-o_{i,k}^{t,n}), \label{eq13}
\end{align}
We can observe that gradient $\partial L/\partial o_{i,k}^{t,n}$ and $\partial L/\partial u_{i,k}^{t,n}$ come from two directions: SD (the left part in Eq.~(\ref{eq12}), (\ref{eq13})) and TD (the right part in Eq.~(\ref{eq12}), (\ref{eq13})). Gradient $\partial L/\partial \hat{o}_{i}^{t,n}$ comes from SD. Finally, we can obtain the gradients of weights $\textbf{w}^n$ and biases $\textbf{b}^n$ as follows
\begin{align}
\label{eq14}
\frac{\partial L}{\partial \textbf{w}^n}&=\sum_{t=1}^T\sum_{k=1}^K\frac{\partial L}{\partial \textbf{u}_k^{t,n}}
\frac{\partial \textbf{u}_k^{t,n}}{\partial \textbf{w}^n}
=\sum_{t=1}^T(\sum_{k=1}^K\frac{\partial L}{\partial \textbf{u}_k^{t,n}})\hat{\textbf{o}}^{{t,n-1}^T}, \\
\label{eq15}
\frac{\partial L}{\partial \textbf{b}^n}&=\sum_{t=1}^T\sum_{k=1}^K\frac{\partial L}{\partial \textbf{u}_k^{t,n}}
\frac{\partial \textbf{u}_k^{t,n}}{\partial \textbf{b}^n}
=\sum_{t=1}^T(\sum_{k=1}^K\frac{\partial L}{\partial \textbf{u}_k^{t,n}}),
\end{align}
where $T$ is the number of timesteps. Due to the non-differentiable property of spiking activity, $\partial o_k/\partial u_k$ cannot be derived. To solve this problem, we adopt the rectangular function $h_k(u_k)$~\cite{Wu2018} to approximate the derivative of spike activity, which is defined by
\begin{equation}
\label{eq16}
\frac{\partial o_k}{\partial u_k} \approx h_k(u_k)=\frac{1}{a}sign(|u_k-{V_{th}}_k|<\frac{a}{2}),
\end{equation}
where $a$ is the width parameter of the rectangular function.

\begin{figure}[t]
\centering
\includegraphics[width=0.95\columnwidth]{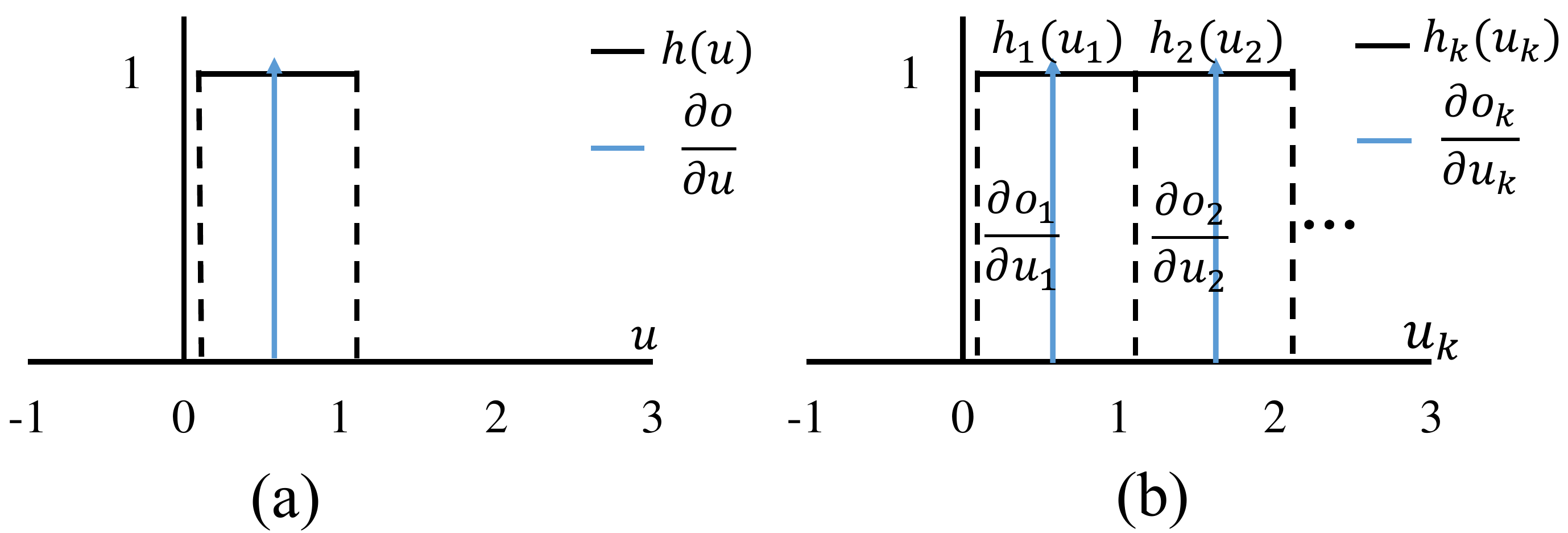} 
\caption{(a) The distribution of approximate derivative of one-level firing. (b) The distribution of approximate derivatives of multi-level firing.}
\label{fig_derivative}
\end{figure}

Considering the gradient propagation in SD from $n$-th layer to $(n-1)$-th layer, the spatial propagation link can be described as: $\partial L/\partial \hat{o}_{i}^{t,n} \rightarrow (\partial L/\partial o_{i,1}^{t,n},...,\partial L/\partial o_{i,K}^{t,n}) \rightarrow (\partial L/\partial u_{i,1}^{t,n},...,\partial L/\partial u_{i,K}^{t,n}) \rightarrow \partial L/\partial \hat{o}_{i}^{t,n-1}$. If it is only one-level firing $(K=1)$, the model will become the standard STBP model. In this case, numerous neurons will fall into the saturation area outside the rectangular area, some of which will become dormant unit, shown in Fig.~\ref{fig_derivative}(a), and the corresponding $\partial o_{i,1}^{t,n}/\partial u_{i,1}^{t,n}$ will be zero due to the limited width of the approximate derivative. Consequently, most of $\partial L/\partial u_{i,1}^{t,n}$ will lose the gradients propagated from SD, and the spatial propagation links from $n$-th layer to $(n-1)$-th layer will be broken, which makes the gradient propagation blocked in SD.

In the case of multi-level firing ($K>1$) and non-overlapping distribution of each level $h_k(u_k)$ where ${V_{th}}_{k+1}-{V_{th}}_{k}=a$, MLF units are less likely to fall into the dormant state after receiving inputs because of the wider non-zero area, as shown in Fig.~\ref{fig_derivative}(b). As a result, MLF can guarantee efficient gradient propagation through the spatial propagation links unless all levels of $u_k$ fall into the areas outside the corresponding rectangular areas, and the adjustment of $\textbf{w}^n$ and $\textbf{b}^n$ can be accelerated during the training process.

The pseudo code for the overall training of the forward and backward process is shown in Appendix \ref{appendix_pseudo}.
\subsection{Dormant-Suppressed Residual Network}
Residual network (ResNet)~\cite{He2016}, as one of the most widely used basic blocks, has achieved great success in deep networks. To convert ResNet into spiking ResNet, we replace BN and ReLU by tdBN~\cite{Zheng2021} and MLF units respectively, where tdBN is used to coordinate distribution difference and normalize the input distribution to $N(0, {(V_{th_1})}^2)$. In spiking ResNet, MLF activation is after the addition of the shortcut connection, as shown by the dotted line in Fig.~\ref{fig_resnet}. The addition will increase the values of the feature map before activation. In SNNs, the increase of the values will make inputs exceed the right side of the rectangular area in Fig.~\ref{fig_derivative} resulting in more dormant units. Besides, due to the discrete property of the activation function, the shortcut connection cannot perform identity mapping well.
Therefore, spiking ResNet still suffers from network degradation, which prevents the directly trained SNNs from going deeper.

\begin{figure}[tb]
\centering
\includegraphics[width=0.85\columnwidth]{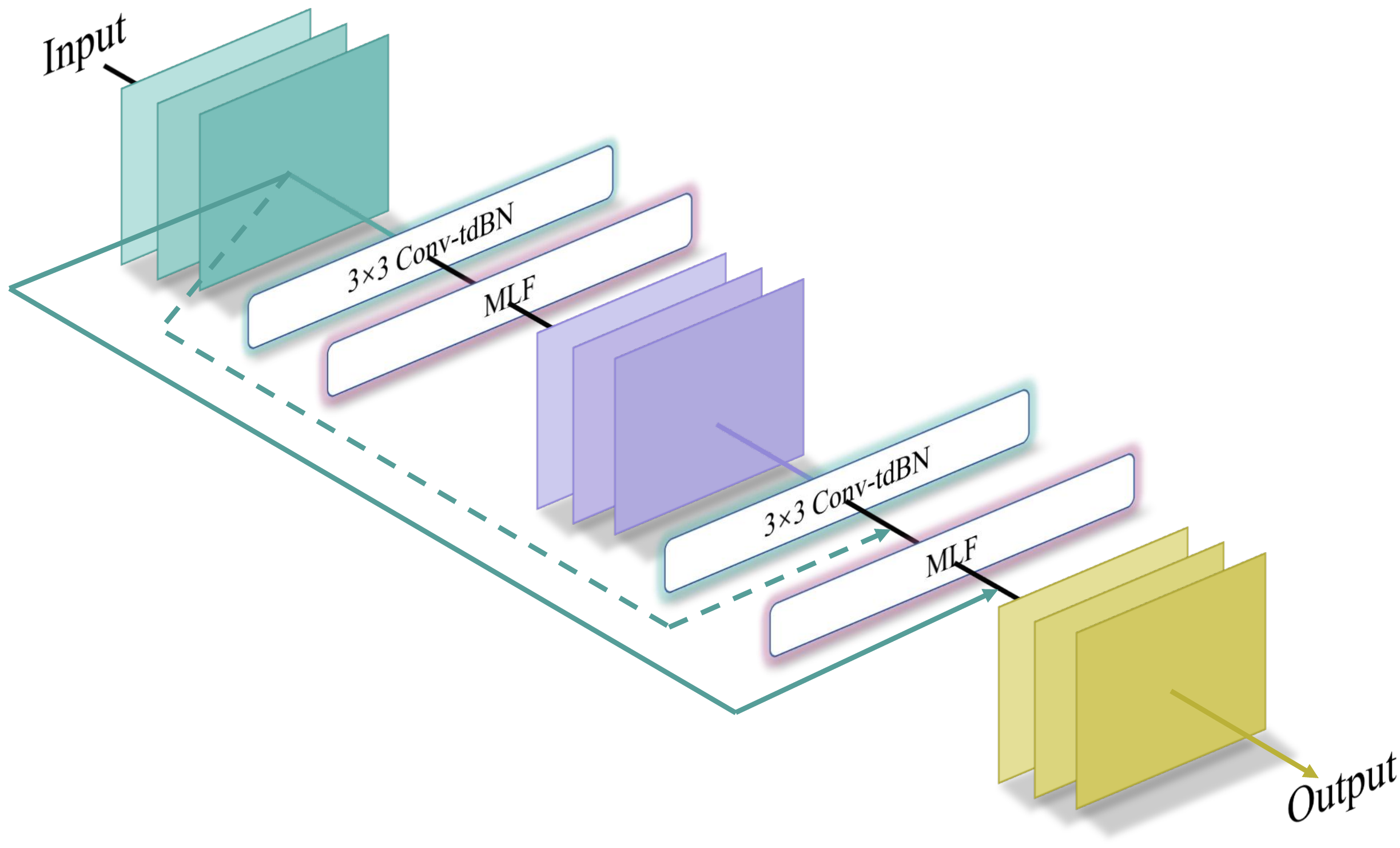} 
\caption{Illustration of spiking ResNet and spiking DS-ResNet. The dotted line connection represents spiking ResNet, and the solid line connection represents spiking DS-ResNet.}
\label{fig_resnet}
\end{figure}

To address the above two issues, we replace activation after addition structure with activation before addition structure and propose the spiking dormant-suppressed residual network (spiking DS-ResNet), as shown by the solid line in Fig.~\ref{fig_resnet}. Spiking DS-ResNet can reduce the proportion of dormant units at the output of each block, enabling more efficient gradient propagation and more efficient identity mapping of discrete spikes in very deep SNNs. We formalize this improvement in theorem~\ref{theorem1}.

\begin{theorem}
\label{theorem1}
Considering the values of feature map normalized by tdBN satisfy $x \sim N(0, {(V_{th_1})}^2)$, the probabilities of a spike from shortcut connection leading to dormant unit in spiking DS-ResNet and spiking ResNet are $P^*$ and $P$ respectively, and the abilities of identity mapping of spiking DS-ResNet and spiking ResNet are $I^*$ and $I$ respectively, then we have $P^* < P$ and $I^* > I$.
\end{theorem}
\begin{proof}
The proof of Theorem 1 is presented in Appendix \ref{appendix_DS_resnet}.
\end{proof}

In this way, dormant units can be effectively suppressed and the shortcut connection can retain the ability of identity mapping, so as to further solve the degradation problem.
\section{Experiments}
The source code of our MLF and spiking DS-ResNet implementation\footnote{\url{https://github.com/langfengQ/MLF-DSResNet}.} is available online.
\subsection{Experimental Settings}
The basic network architecture of our experiments is ResNet. The first layer is $3\times3$ convolutions as the encoding layer. Similar to ResNet for CIFAR10 in~\cite{He2016}, we start stacking residual block, which contains 2 convolution layers, to $6N$ layers with $2N$ layers for each feature map size. The number of channels will be doubled if the feature map is halved. We set three different initial channels (small, middle, large) for the first layer, the numbers of which are 16, 32, 64 respectively. ResNet ends with global average pooling and a 10/11 fully-connected classifier. The total layers of our ResNet are $6N+2$. More detailed network architecture and other experimental settings are summarized in Appendix \ref{appendix_setting}.
\begin{table*}[t]
\centering
\begin{tabular}{c c c r c r c r}
    \toprule
     Model & Method & \multicolumn{2}{c}{CIFAR10} & \multicolumn{2}{c}{DVS-Gesture} & \multicolumn{2}{c}{CIFAR10-DVS} \\ \midrule
      &  & Acc. & Params & Acc.  & Params & Acc. & Params\\ \midrule
    \cite{Amir2017} & TrueNorth & - & \multicolumn{1}{c}{-} & 94.59 & 18.99M & - & \multicolumn{1}{c}{-} \\
    \cite{Sengupta2019} & ANN-SNN & 91.55 & 33.63M & - & \multicolumn{1}{c}{-} & - & \multicolumn{1}{c}{-} \\
    \cite{Wu2019} & STBP & 90.53 & 44.99M & - & \multicolumn{1}{c}{-} & 60.50 & 26.82M  \\
    \cite{He2020} & STBP & - & \multicolumn{1}{c}{-} & 93.40 & 2.32M & - & \multicolumn{1}{c}{-} \\
    \cite{Kugele2020} & ANN-SNN & - & \multicolumn{1}{c}{-} & 95.56 & 0.80M & 65.61 & 0.50M \\
    \cite{Lee2020} &  Spike-based BP & 90.95 & 18.20M & - & \multicolumn{1}{c}{-} & - & \multicolumn{1}{c}{-} \\
    \cite{Wu2021} & ASF-BP & 91.35 & 8.81M & - & \multicolumn{1}{c}{-} & 62.50 & 26.78M\\
    \cite{Zheng2021} & STBP-tdBN & 93.16 & 15.10M & 96.87 & 3.50M & 67.80 & 16.27M \\
    \cite{Yan2021} & ANN-SNN & 94.16 & 9.33M & - & \multicolumn{1}{c}{-} & - & \multicolumn{1}{c}{-}\\ \midrule
    Our model & MLF ($K=3$) + spiking DS-ResNet & 94.25 & 4.32M & 97.29 & 0.27M & 70.36 & 0.69M\\\bottomrule
\end{tabular}
\caption{Comparison of different methods on three datasets. The unit of Acc. is \%.}
\label{tab_acc}
\end{table*}
\subsection{Classification Accuracy}
\paragraph{CIFAR10}
For CIFAR10, we apply spiking DS-ResNet (20-layer, large) to evaluate the average performance of our model in 4 timesteps. The level of MLF is set to 3. The performance is averaged over 5 runs.

Table~\ref{tab_acc} shows the comparison of our results and existing state-of-the-art results on CIFAR10. We notice that~\cite{Yan2021} reported a competitive accuracy on CIFAR10, whereas it is a conversion-based method and requires a large number of timesteps to ensure encoding resolution. Our model achieves state-of-the-art performance (94.25\%) with fewer timesteps and fewer trainable parameters.

\paragraph{DVS-Gesture}
DVS-Gesture~\cite{Amir2017} is a neuromorphic vision dataset with more temporal information. We apply spiking DS-ResNet (20-layer, small) to evaluate the average performance of our model with MLF levels of 3. The performance is averaged over 5 runs.

As we can see from Table~\ref{tab_acc},~\cite{Zheng2021} have better performance on DVS-Gesture with directly-trained SNNs compared with~\cite{He2020}. The reason is that the former succeeded in the direct training of large-size and deep SNNs. However, they have to apply a large network structure to get a competitive accuracy, due to a large number of dormant units and the weak expression ability of neurons. We achieve more efficient training and get state-of-the-art performance (97.29\%) with a smaller network structure, the initial channel number of which is only 16.

\paragraph{CIFAR10-DVS}
CIFAR10-DVS~\cite{Li2017} is a more challenging and easy-overfitted neuromorphic vision dataset. We apply spiking DS-ResNet (14-layer, middle) to evaluate the average performance of our model with MLF levels of 3. The performance is averaged over 5 runs.

As indicated in Table~\ref{tab_acc},~\cite{Wu2021} ignored the information in TD when training SNNs, resulting in their accuracy not being as good as~\cite{Zheng2021}, which demonstrates the advantages of the spike-based direct training method in dealing with spatio-temporal information. On this basis, our model further improves the performances of directly trained SNNs and gets state-of-the-art performance (70.36\%) with much smaller network architecture.
\subsection{Analysis of Our Model}
\subsubsection{Effects of Level \textit{K} and Spiking DS-ResNet}
The performance improvement of our model benefits from two aspects: MLF and spiking DS-ResNet. We design experiments to explicitly analyze the effects of the level of MLF and spiking DS-ResNet. We apply ResNet (20-layer, large) as basic architecture on CIFAR10 in 4 timesteps with various methods.

\begin{table}[tb]
\centering
\begin{tabular}{c | c}
    \toprule
    Method & Acc. (\%) \\\midrule
    Spiking ResNet($K=1$) & 92.55 \\
    ResNet-SNN($K=1$) & 93.04  \\
    Spiking DS-ResNet($K=1$) & 93.54 \\
    Spiking DS-ResNet($K=2$) & 94.13 \\
    Spiking DS-ResNet($K=3$) & 94.25 \\

    \bottomrule
\end{tabular}
\caption{Effects of Level $K$ and Spiking DS-ResNet.}
\label{tab_gradientvanishing}
\end{table}

As shown in Table~\ref{tab_gradientvanishing}, ResNet-SNN~\cite{Zheng2021} has better performance than spiking ResNet for the reason that ResNet-SNN can efficiently control the generation of the dormant units by adding tdBN to every shortcut connection. However, it cannot efficiently perform identity mapping and will increase the number of parameters. Our spiking DS-ResNet can take both dormant unit and identity mapping into account, therefore having better performance.
Besides, MLF shows prominent advantages in the improvement of accuracies, and a larger $K$ will benefit the performance progressively.
While, the improvement from $K=1$ to $K=2$ is greater than that from $K=2$ to $K=3$ for the reason that the data falling into $h_2(u_2)$ is much more than the data falling into $h_3(u_3)$ considering that the input distribution satisfies $N(0, {(V_{th_1})}^2)$.
In our experiments, MLF with level 3 is enough to cover most of the sharp features (see Appendix \ref{appendix_level}).
\begin{figure*}[t]
    \centering
    \setcounter{figure}{4}
    \includegraphics[width=1\textwidth]{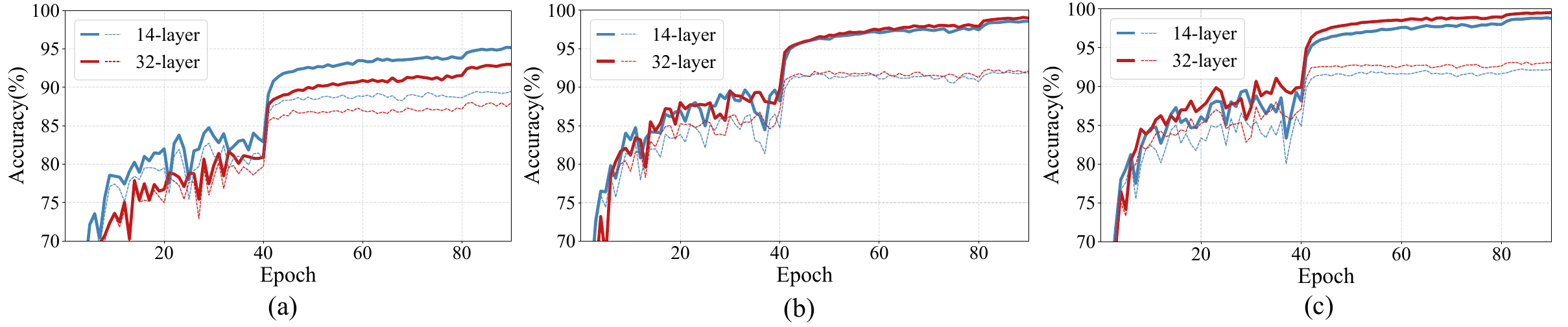} 
    \caption{Training on CIFAR10. Thin dotted curves denote testing accuracy, and bold solid curves denote training accuracy. (a) ResNet-SNN without MLF. (b) ResNet-SNN with MLF. (c) Spiking DS-ResNet with MLF.}
    \label{fig_goingdeeper}
    \end{figure*}
\begin{figure}[t]
\centering
\setcounter{figure}{3}
\includegraphics[width=0.95\columnwidth]{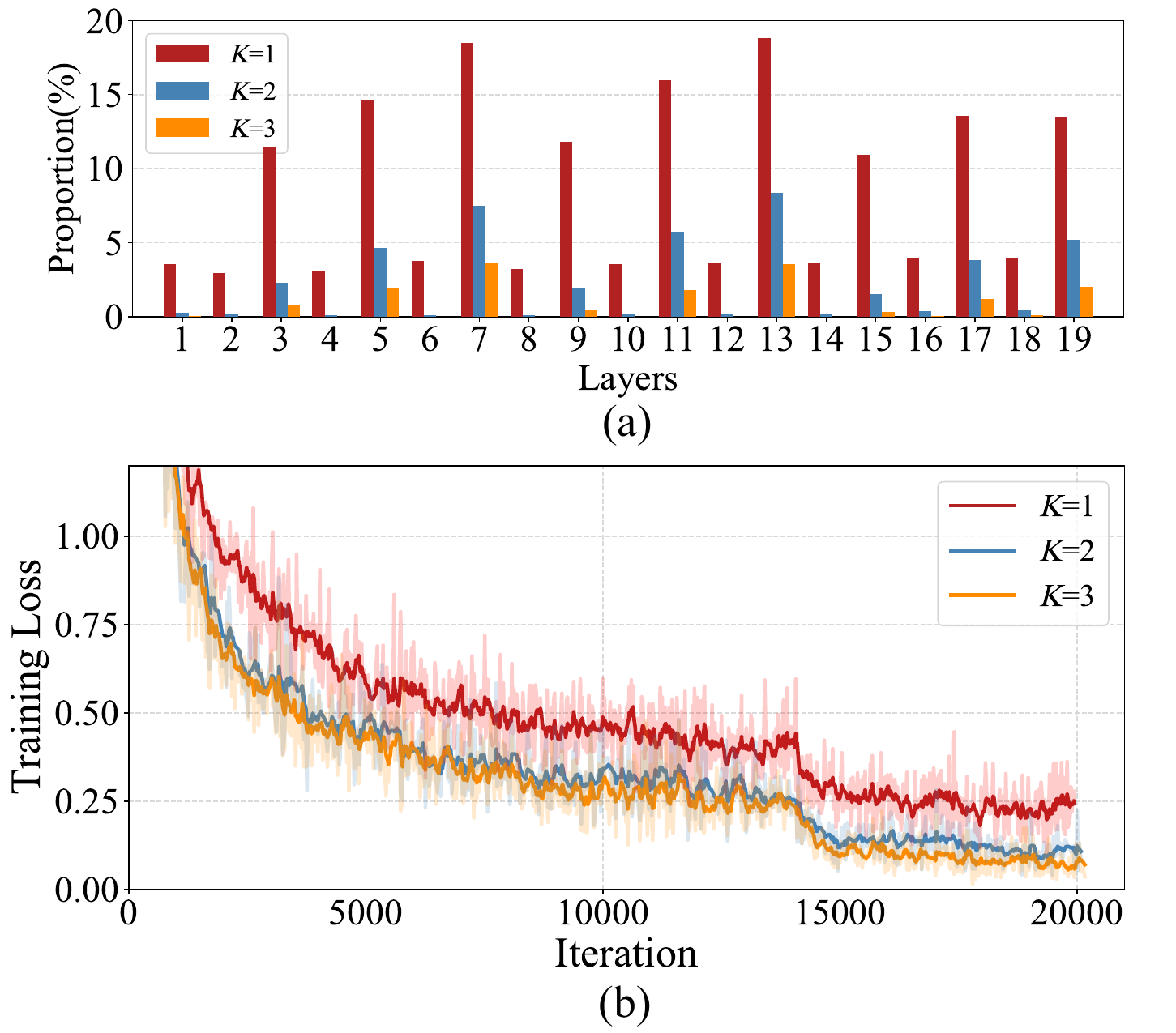} 
\caption{(a) The average proportion of dormant units in each layer. (b) The training loss during the whole training process.}
\label{fig_gradientvanishing}
\end{figure}

\subsubsection{MLF for Gradient Vanishing}
In this part, we conduct experiments to demonstrate that MLF can effectively alleviate the blocking of gradient propagation in SD. We apply the spiking ResNet (20-layer, middle) on CIFAR10 with levels of 1, 2, 3 in 4 timesteps.

As indicated in Fig.~\ref{fig_gradientvanishing}(a), there is a large proportion of dormant units during the training process when $K=1$ (without MLF), which will lead to the gradient vanishing and prevent further enhancement of sharp features. While, MLF ($K>1$) has prominent advantages in reducing the proportion of dormant units, and the gradient vanishing can be alleviated. In this way, MLF can increase the gradient of trainable parameters (see Appendix \ref{appendix_increase}), accelerate the convergence speed of network training, and significantly improve the performance of directly trained SNNs, as shown in Fig.~\ref{fig_gradientvanishing}(b).

More analysis of gradient vanishing is presented in Appendix \ref{appendix_vanishing}.

\subsubsection{Going Deeper}
In directly trained deep SNNs, the serious degradation problem greatly restricts SNNs to a shallow. In this part, we take ResNet (\textit{x}-layer, middle) as basic architecture to conduct multiple experiments on very deep SNNs to demonstrate the ability of our model to solve the degradation problem. We train various models and take the reported deepest network architecture (ResNet-SNN) as the baseline. The level of MLF is set to 3. We record the training and testing accuracy of the whole training process.

As depicted in Fig.~\ref{fig_goingdeeper}(a), ResNet-SNN without MLF have experienced serious network degradation problem only at 32-layer (the deeper 32-layer has lower training/testing accuracy than the shallow 14-layer). After introducing MLF, the degradation problem is alleviated, as shown in Fig.~\ref{fig_goingdeeper}(b). At the same time, MLF effectively improves the training/testing accuracy of both 14-layer and 32-layer. Finally, with spiking DS-ResNet and MLF, the degradation problem is further solved, as shown in Fig.~\ref{fig_goingdeeper}(c).

Moreover, we explore deeper networks to test our model on CIFAR10, and we achieve very deep SNNs (68-layer) without degradation, which validates that our model can efficiently solve the degradation problem in deep SNNs. The testing accuracy is summarized in Table~\ref{tab_goingdeeper}.

\begin{table}[tb]
\centering
\begin{tabular}{c@{\,}|@{ \,} c @{\, \,}c @{\, \,}c @{\, \,}c @{\, \,} c}
    \toprule
    Layer&14-layer&20-layer&32-layer&44-layer&68-layer\\ \midrule
    Acc.(\%) & 92.46 & 92.95 & 93.34 & 93.42 & 93.48 \\ \bottomrule
\end{tabular}
\caption{Accuracy on CIFAR10 with spiking DS-ResNet and MLF.}
\label{tab_goingdeeper}
\end{table}

\section{Conclusion}
In this paper, we have proposed the MLF method based on STBP and spiking DS-ResNet for direct training of deep SNNs to combat the gradient vanishing and network degradation caused by the limitation of the binary and non-differentiable properties of spike activities. We prove that MLF can expand the non-zero area of the approximate derivatives by allocating the coverage of the approximate derivative of each level, so as to reduce dormant units and make the gradient propagation more efficient in deep SNNs. Besides, an MLF unit can generate spikes with different thresholds when activating the input, which can improve its expression ability. Spiking DS-ResNet can reduce the probability of dormant unit generation making it more suitable for gradient propagation and can efficiently perform identity mapping of discrete spikes in very deep SNNs. With MLF and spiking DS-ResNet, our model achieves state-of-the-art performances with fewer parameters on both non-neuromorphic and neuromorphic datasets compared with other SNN models and makes SNNs go very deep without degradation. This paper provides an efficient solution to gradient vanishing and network degradation in the directly trained SNNs enabling SNNs to go deeper with high performance.

\section*{Acknowledgments}
This work is supported by China Brain Project (2021ZD0200400), Natural Science Foundation of China (No. 61925603), the Key Research and Development Program of Zhejiang Province in China (2020C03004), and Zhejiang Lab.

\bibliographystyle{named}
\bibliography{ijcai22}

\appendix
\numberwithin{equation}{section}
\numberwithin{figure}{section}
\numberwithin{table}{section}
\numberwithin{algorithm}{section}
\onecolumn
\subsection*{
    \centerline{\LARGE Appendices}
}

\setcounter{tocdepth}{2}
\section{Equivalence Proof of Union}\label{appendix_equivalent}
The final output of MLF unit is the union of the spikes fired by all level neurons in Fig.~\ref{fig_MLF}. We can split the union and directly connect the LIF neurons to the next MLF unit through a shared weight, which is biologically plausible, as shown in Fig.~\ref{fig_bio_plasti_explan}. The pink weight $w_{i1}^{n+1}$ is the shared weight. This connection in Fig.~\ref{fig_bio_plasti_explan} is completely equivalent to the MLF connection in Fig.~\ref{fig_MLF}. Therefore, pre-synaptic input $x_i^{t,n}$ can be computed as
\begin{equation}
\begin{aligned}
\label{bio_plasti_explan}
x_{i}^{t,n}&=\sum_{j}\textbf{w}_{ij}^{n}(\textbf{o}_i^{t,n-1})^T+b_i^n \\
&=\sum_{j}(w_{ij}^{n}, ..., w_{ij}^{n})(o_{i,1}^{t,n-1},..,o_{i,K}^{t,n-1})^T+b_i^n \\
&=\sum_{j}(w_{ij}^{n}o_{i,1}^{t,n-1}+...+w_{ij}^{n}o_{i,K}^{t,n-1})+b_i^n \\
&=\sum_{j}w_{ij}\sum_{k}o_{i,k}^{t,n-1}+b_i^n\\
&=\sum_{j}w_{ij}s(\textbf{o}_i^{t,n-1})+b_i^n,
\end{aligned}
\end{equation}
where $\textbf{w}_{ij}^{n}=(w_{ij}^{n}, w_{ij}^{n}, ..., w_{ij}^{n})$ denotes the vector of shared weights. Thus, Eq. (\ref{bio_plasti_explan}) is equivalent to Eq. (\ref{eq8}) plus Eq. (\ref{eq1}).

\begin{figure}[h]
\centering
\includegraphics[width=0.5\columnwidth]{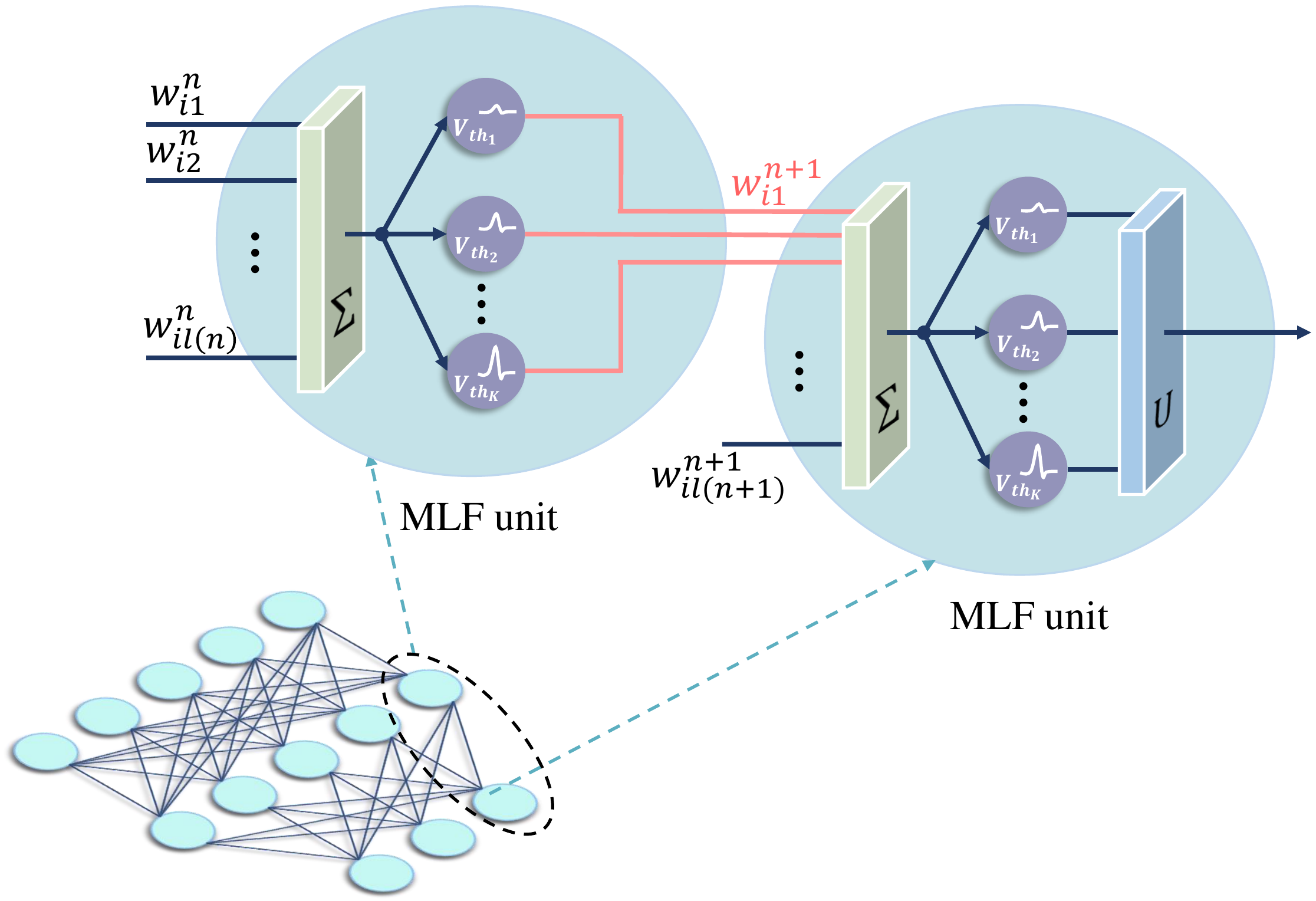} 
\caption{Illustration of the equivalent structure of MLF unit. LIF neurons in a MLF unit share a weight $w_{i1}^{n+1}$.}
\label{fig_bio_plasti_explan}
\end{figure}

\section{The Overall Training}\label{appendix_pseudo}
Our experiments are based on spiking convolution neural network, so we first rewrite Eq. (\ref{eq1}) to the convolution form as follows,
\begin{equation}
\label{overalltraining}
x^{t,n}=w_{conv}^n \circledast \hat{o}^{t,n-1}+b_{conv}^n,
\end{equation}
where $\circledast$ denotes the convolution operation. $x^{t,n}\in R^{N \times C \times H \times W}$ is the output of the $n$-th convolution layer at time $t$, where $N$, $C$, $H$ and $W$ denote batch size, channel number, height and width respectively. $w_{conv}^n$ and $b_{conv}^n$ represent convolution kernel and bias of the $n$-th convolution layer respectively.

The pseudo code for the overall training of our model in one iteration is shown in Algorithm \ref{alg_overaltraining}.
\begin{algorithm}[H]
\caption{Overall training in one iteration.}
\label{alg_overaltraining}
\textbf{Input}: network input $X^t(t=1,2,...,T)$, class label $Y$.  \\
\textbf{Output}: the updated parameters $\textbf{w}^n$ and $\textbf{b}^n$ of the network.
\begin{algorithmic}[1] 
\Function{MLF}{$x^{n}$}
\For{$k=1$ to $K$}
\State $u^{1,n}_k=x^{1,n}$
\EndFor
\State $\textbf{o}^{1,n}=f(\textbf{u}^{1,n}-\textbf{V}_{th})$
\State $\hat{o}^{1,n}=s(\textbf{o}^{1,n})$
\For{$t=1$ to $T$}
\State $\textbf{u}^{t+1,n}=k_{\tau}\textbf{u}^{t,n}\circ(1-\textbf{o}^{t,n})+x^{t+1,n}$
\State $\textbf{o}^{t+1,n}=f(\textbf{u}^{t+1,n}-\textbf{V}_{th})$
\State $\hat{o}^{t+1,n}=s(\textbf{o}^{t+1,n})$
\EndFor
\State \Return $\hat{o}^{n}$
\EndFunction

\Statex \Statex \textbf{Forward:}
\For{$t=1$ to $T$}
\State $x^{t,1}=w_{conv}^n \circledast X^{t}+b_{conv}^n$
\EndFor
\State $y^1 \gets$ tdBN$(x^{1})$ \quad // \cite{Zheng2021}
\State $\textbf{o}^1 \gets$ MLF$(y^{1})$
\For{$n=2$ to $N$}
\For{$t=1$ to $T$}
\State $x^{t,n}=w_{conv}^n \circledast \hat{o}^{t,n-1}+b_{conv}^n$
\EndFor
\State $y^n \gets$ tdBN$(x^{n})$ \quad // \cite{Zheng2021}
\State $\hat{o}^n \gets$ MLF$(y^{n})$
\EndFor
\State $Q \gets$ Decoding$(\hat{o}^{N-1})$
\State $L \gets$ CrossEntropy$(Q, Y)$ 

\Statex \Statex \textbf{Backward:}
\State Gradient zero initialization
\For{$t=T$ to $1$}
\State $\frac{\partial L}{\partial \hat{o}^{t, N-1}} \gets$ GradientCalculation$(L)$ 
\EndFor
\For{$n=N-1$ to $1$}
\For{$t=T$ to $1$}
\For{$k=K$ to $1$}
\State $\frac{\partial L}{\partial u^{t, n}_k} \gets$ GradientBackward$(\frac{\partial L}{\partial \hat{o}^{t, n}},$ $\frac{\partial L}{\partial u^{t+1, n}_k})$ \quad // Eq. (\ref{eq12}), (\ref{eq13})
\EndFor
\State $\frac{\partial L}{\partial \hat{o}^{t, n-1}} \gets$ GradientBackward$(\frac{\partial L}{\partial u^{t, n}_1}, \frac{\partial L}{\partial u^{t, n}_2},..., \frac{\partial L}{\partial u^{t, n}_K})$ \quad // Eq. (\ref{eq12_new})
\EndFor
\EndFor
\State Update parameters $\textbf{w}^n$ and $\textbf{b}^n$ \quad // Eq. (\ref{eq14}), (\ref{eq15})
\end{algorithmic}
\end{algorithm}

\section{Proof of Theorem 1}\label{appendix_DS_resnet}
\begin{theorem} 
Considering the values of the feature map normalized by tdBN satisfy $x \sim N(0, {(V_{th_1})}^2)$, the probabilities of a spike from shortcut connection leading to the dormant unit in spiking DS-ResNet and spiking ResNet are $P^*$ and $P$ respectively, and the abilities of identity mapping of spiking DS-ResNet and spiking ResNet are $I^*$ and $I$ respectively, then we have $P^* < P$ and $I^* > I$.
\end{theorem}
\begin{proof}
For convenience, we assume that the feature maps of the residual connection and shortcut connection are independent of each other and do not consider the residual membrane potential. We consider the case of $K = 1$, so the gradient-available interval is $[V_{th_1} - a/2,V_{th_1} + a/2]$. For standard spiking ResNet, a spike from the shortcut connection will lead to the dormant unit if the following equation is satisfied.
\begin{align}
1 + x &> V_{th_1} + a/2,
\end{align}
where $x\sim N(0, {(V_{th_1})}^2)$ is the feature map, $1$ is the spike from shortcut connection. Therefore, the probability of this event is
\begin{align}
P = p(x &> V_{th_1} + a/2 - 1) = 1-\Phi(\frac{V_{th_1} + a/2 - 1}{V_{th_1}}),
\end{align}
where $\Phi$ is the cumulative distribution function (CDF) of standard normal distribution. For spiking DS-ResNet, the probability of a spike from the shortcut connection leading to the dormant unit satisfies
\begin{align}
P^* = p(x &> V_{th_1} + a/2) = 1-\Phi(\frac{V_{th_1} + a/2}{V_{th_1}}).
\end{align}
Thus, we have $P^* < P$. In experiment, we set $a=1$ and $V_{th_1}=0.6$, so the $P^*$ is much less than $P$.

At present, there is no unified standard for quantifying identity mapping ability. Here, we measure it by considering whether the spikes of the shortcut connection can be reflected in the output of a residual block. Specifically, if the output of a residual block is not affected by the shortcut connection, the spikes of the shortcut connection cannot be reflected in the output. For spiking ResNet, if $x<V_{th_{1}}-1$, the output will always be $0$ whether there is a spike from the shortcut connection or not. Similarly, if $x>V_{th_{1}}$, the output will always be $1$ whether there is a spike from the shortcut connection or not. Thus, in the above two cases, the spikes of the shortcut connection cannot be mapped to the output. On the contrary, we consider that the input of a residual block can be mapped to the output through the shortcut connection in case of 
\begin{align}
V_{th_1} - 1 < x < V_{th_1}.
\end{align}
The probability of this event is
\begin{align}
p(V_{th_1} - 1 < x < V_{th_1}) = \Phi(\frac{V_{th_1}}{V_{th_1}}) - \Phi(\frac{V_{th_1} - 1}{V_{th_1}}).
\end{align}
In practical application, we usually set $V_{th_1}=0.6$, which means there will be about 41\% input spikes that cannot be mapped to the output.
While for spiking DS-ResNet, all input spikes can be directly mapped to the output. Therefore we have $I^* > I$.

\end{proof}

\section{Details of Experiments}\label{appendix_setting}
\subsection{Dataset Introduction}
We conduct the experiments on a non-neuromorphic dataset and two neuromorphic datasets.
\subsubsection{CIFAR10}
CIFAR10 is a color image dataset widely used for identifying universal objects. It contains 50,000 training images and 10,000 testing images in 10 classes with size of $32 \times 32$.

For the data pre-processing of CIFAR10, we follow the standard data augmentation strategy. The original images are first randomly cropped and flipped, and then normalized by subtracting the global mean value of pixel intensity and divided by the global standard variance along each RGB channel.

\subsubsection{DVS-Gesture}
DVS-Gesture \cite{Amir2017} is a neuromorphic vision dataset, which is obtained by capturing the different hand gestures from 29 subjects under 3 illumination conditions. It is worth noting that each recorded sample in the DVS-Gesture dataset contains two gestures belonging to class 8 and we use both of them for training and testing. It contains 1,176 event streams from 23 subjects for training and 288 event streams from 6 subjects for testing, which belong to 11 classes with the size of $128 \times 128$.

For the data pre-processing of DVS-Gesture, we downsample the original event streams to the size of $32 \times 32$ and sample a slice every 30ms. In each timestep, the input data is only one slice. We set 40 timesteps for training and testing, which means we use the first 1.2s of each event stream.

\subsubsection{CIFAR10-DVS}
CIFAR10-DVS \cite{Li2017} is a neuromorphic vision dataset obtained by displaying the moving images of the CIFAR10 dataset on a monitor. It is a more challenging and easy-overfitted neuromorphic vision dataset due to the noisy environment, a small number of samples, and the large intra-class variance. It consists of 10,000 event streams in 10 classes with the size of $128 \times 128$. We randomly selected 9,000 of the event streams for training and the rest for testing.

For the data pre-processing of CIFAR10-DVS, we downsample the original event streams to the size of $42 \times 42$ and sample a slice every 10ms. We set 10 timesteps for training and testing, which means we use the first 100ms of each event stream.

\subsection{Detailed Experiment Settings}
\subsubsection{Architectures of ResNet}
The basic network architecture of our experiments is ResNet, which has achieved great success in the application of deep ANNs. To convert the network architecture into the SNN version, we replace BN and ReLU with tdBN and MLF units respectively. The weights of the networks are initialized according to the normal distribution \cite{HeK2015}. The detailed architectures of ResNet are summarized in Table~\ref{tab_resnet}.

\begin{table}[h]
\centering
\begin{tabular}{c |c }
    \toprule
    Layer & ($6N+2$)-layer, small/middle/large \\ \midrule
    conv1 & $3\times3,16/32/64$\\ \midrule
    conv2\_x & $\left(\begin{array}{l}3\times3,16/32/64\\3\times3,16/32/64\\\end{array}\right)\times 3$ \\ \midrule
    conv3\_x &$\left(\begin{array}{l}3\times3,32/64/128\\3\times3,32/64/128\\\end{array}\right)\times 3$ \\ \midrule
    conv4\_x &$\left(\begin{array}{l}3\times3,64/128/256\\3\times3,64/128/256\\\end{array}\right)\times 3$ \\ \midrule
    ~ & average pool, 10 (11)-d fc \\ \bottomrule
\end{tabular}
\caption{ Architectures of ResNet. The strides of conv2\_1, conv3\_1, and conv4\_1 are set to 2 for downsampling, and the other strides are set to 1. $N$ represents the number of the blocks stacked for each feature map size.}
\label{tab_resnet}
\end{table}

\subsubsection{Hyper-Parameter}
For all experiments, the delay factors $k_{\tau}$ and the width parameter $a$ of the rectangular function are set to $0.25$ and $1$ respectively. The $1$st level threshold of MLF units is $0.6$ and the $k$-th level threshold is $0.6+k*a$, which means $\textbf{V}_{th}=(0.6,1.6,2.6,...)$.

\subsubsection{Encoding and Decoding}
For the spike encoding method, we follow the approach by \cite{Wu2019,Zheng2021} and take the first layer as the encoding layer. This encoding method can encode a pixel to multiple locations and channels, so it doesn't have much demand on timestep.

For the decoding layer and loss function, we follow the approach by \cite{Zheng2021}. The last layer is considered the decoding layer. 

\subsubsection{Optimizer}
For all experiments, we adopt the stochastic gradient descent (SGD) optimizer with an initial learning rate of $0.1$ and momentum of $0.9$. Weight decay is set to 0.0001 for CIFAR10 and DVS-Gesture, and 0.001 for CIFAR10-DVS. The learning rate is divided by 10 every 40 epochs for one-level firing ($K=1$), and divided by 10 every 35 or 40 epochs for MLF ($K>1$). The batch size of CIFAR10, DVS-Gesture and CIFAR10-DVS are 64, 28, and 32 respectively.

\subsubsection{Computing Infrastructure}
We implement our models in the PyTorch framework version 1.9.0 with GPU acceleration. All of our models are trained on one TITAN X GPU, running 64-bit Linux 4.4.0.

\section{Level Setting}\label{appendix_level}
In our experiments, we set the value of each level threshold of $\textbf{V}_{th}$ with the condition of ensuring that the area of each level $h_k(u_k)$ does not overlap each other, shown in Fig.~\ref{fig_levelsetting}. The interval between ${V_{th}}_{(k+1)}$ and ${V_{th}}_{k}$ is at least one width of $h_k(u_k)$. In addition, we adopt the normalization method of tdBN and normalize the feature maps to $N(0, {(V_{th_1})}^2)$. Therefore, three-level firing is enough to cover most of the sharp features with large values. Fig.~\ref{fig_levelsetting} demonstrates the distribution of the membrane potentials in MLF units ($K=3$). It can be seen that few values of the membrane potentials can exceed the right side of $h_3(u_3)$ with the value of $3.1$. In this case, setting more levels cannot bring significant performance improvements. The average accuracy improvement from three-level firing to four-level firing is no more than $0.1\%$ in our experiments.
\begin{figure}[h]
\centering
\includegraphics[width=0.7\columnwidth]{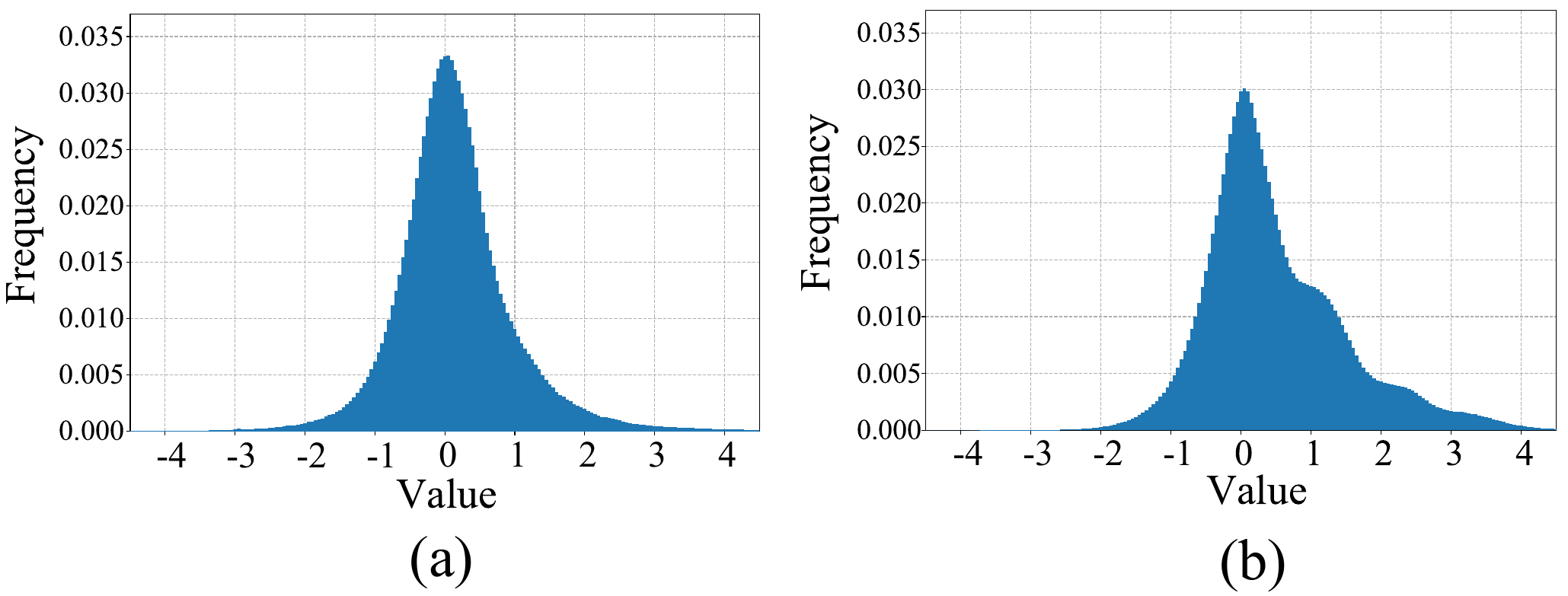} 
\caption{(a) The distribution of the membrane potentials in MLF units ($K=3$). (b) The distribution of the membrane potentials in MLF units ($K=3$) with input added by the shortcut connections.}
\label{fig_levelsetting}
\end{figure}

However, in spiking ResNet, the feature maps normalized by tdBN will be added by the shortcut connections before activation, which will change the distribution. We count the distribution of the membrane potentials in MLF units ($K=3$) with input added by the shortcut connections, as shown in Fig.~\ref{fig_levelsetting}(b). It can be seen that the distribution moves to the right
as a whole and the decline slows down at values 1, 2, and 3, which is in line with the output characteristic of MLF unit with level 3. In this case, some values of membrane potentials can exceed the right side of $h_3(u_3)$. For this, we apply spiking ResNet (20-layer, middle) on CIFAR10 in 2 timesteps with hybrid-level MLF and 3-level MLF. For hybrid-level MLF, we set four-level firing for the feature maps added by shortcut connections and set three-level firing for the feature maps without adding shortcut connections. The performances of hybrid level MLF and 3-level MLF are 92.75\% and 92.71\%, which means the values that exceed the right side of $h_3(u_3)$ have little effect on accuracy.

In summary, MLF with level 3 is enough for the networks if the area of each level $h_k(u_k)$ does not overlap each other, and tdBN is adopted to normalize the feature maps.

\section{The Increasing Gradient}\label{appendix_increase}
For further analysis, we make quantitative statistics for the gradients of convolution weights in spiking ResNet (20-layer, middle) summarized in Table~\ref{tab_grad}. It can be seen that the improvement of gradients in shallow layers is greater than that in deep layers, and the overall improvement of gradients is increased by more than 40\% compared with $K=1$.

\begin{table}[h]
\centering
\begin{tabular}{c c c c c c}
    \toprule
    Method&conv2\_x&conv3\_x&conv4\_x&Imp.\\ \midrule
    $K=1$ & 0.610 & 0.506 & 0.457 & -\\
    $K=2$ & 1.006 & 0.744 & 0.487 & 42.2\%\\
    $K=3$ & 1.051 & 0.754 & 0.493 & 46.1\%\\ \bottomrule
\end{tabular}
\caption{The average convolution weight gradients ($\times10^{-3}$) of the blocks with the same feature map size and the overall improvement of gradients compared with $K=1$.}
\label{tab_grad}
\end{table}

We notice that the gradients in shallow layers are greater than that in deep layers. To explain this phenomenon, we consider one weight $w^n$ at position $(d,i_{k},j_{k},c)$ of the convolution kernel $\textbf{w}^n$ and ignore $T$ and $K$. $d$, $i_{k}$, $j_{k}$ and $c$ denote index of input channel, height, width and output channel respectively. The gradient is as follows
\begin{equation}
\frac{\partial L}{\partial w^n} = \frac{\partial L}{\partial \textbf{u}^{n}}\frac{\partial \textbf{u}^{n}}{\partial w^n}
=\sum_{i=1}^M\sum_{j=1}^M\frac{\partial L}{\partial u_{i,j,d}^{n}}o^{n-1}_{i+i_{k},j+j_{k},c},
\end{equation}
where $M$ is the size of the feature map. As we can see, when $M$ is larger, there are more additional terms, and $|\frac{\partial L}{\partial w^n}|$ tends to be larger. As a result, the gradients in shallow layers with larger feature map sizes are greater than those in deep layers.

\section{Catastrophic Gradient Vanishing}\label{appendix_vanishing}

\subsection{At the Beginning of Training}
If an SNN structure without MLF has no residual structure and has several fully connected layers, one layer may not fire any spikes, and the catastrophic gradient vanishing will easily occur. In this case, there will be no gradient for backpropagation, resulting in non-convergence.

We conduct experiments on VGG16 \cite{Simonyan2014} and CIFAR10 dataset with levels of 1, 2, and 3. As Fig.~\ref{fig_vgg}(b) indicates that ``without MLF'' ($K=1$) cannot enter convergence in the training process as a result of a large number of units falling into dormant unit* in fully-connected layers due to the low membrane potential, as shown in Fig.~\ref{fig_vgg}(a). Especially in the 15th layer, all units are dormant unit*, which means all units will lose the ability of gradient propagation in SD. Consequently, ``without MLF'' cannot enter convergence.
In this case, the hyper-parameter setting and the initialization of network parameters will need to be more elaborate, which greatly limits the generality and usability of the model.

However, MLF can easily solve these problems benefitting from its improved expression ability and its ability to solve gradient vanishing, as shown in Fig.~\ref{fig_vgg}(c), (d).
\begin{figure}[h]
\centering
\includegraphics[width=0.45\columnwidth]{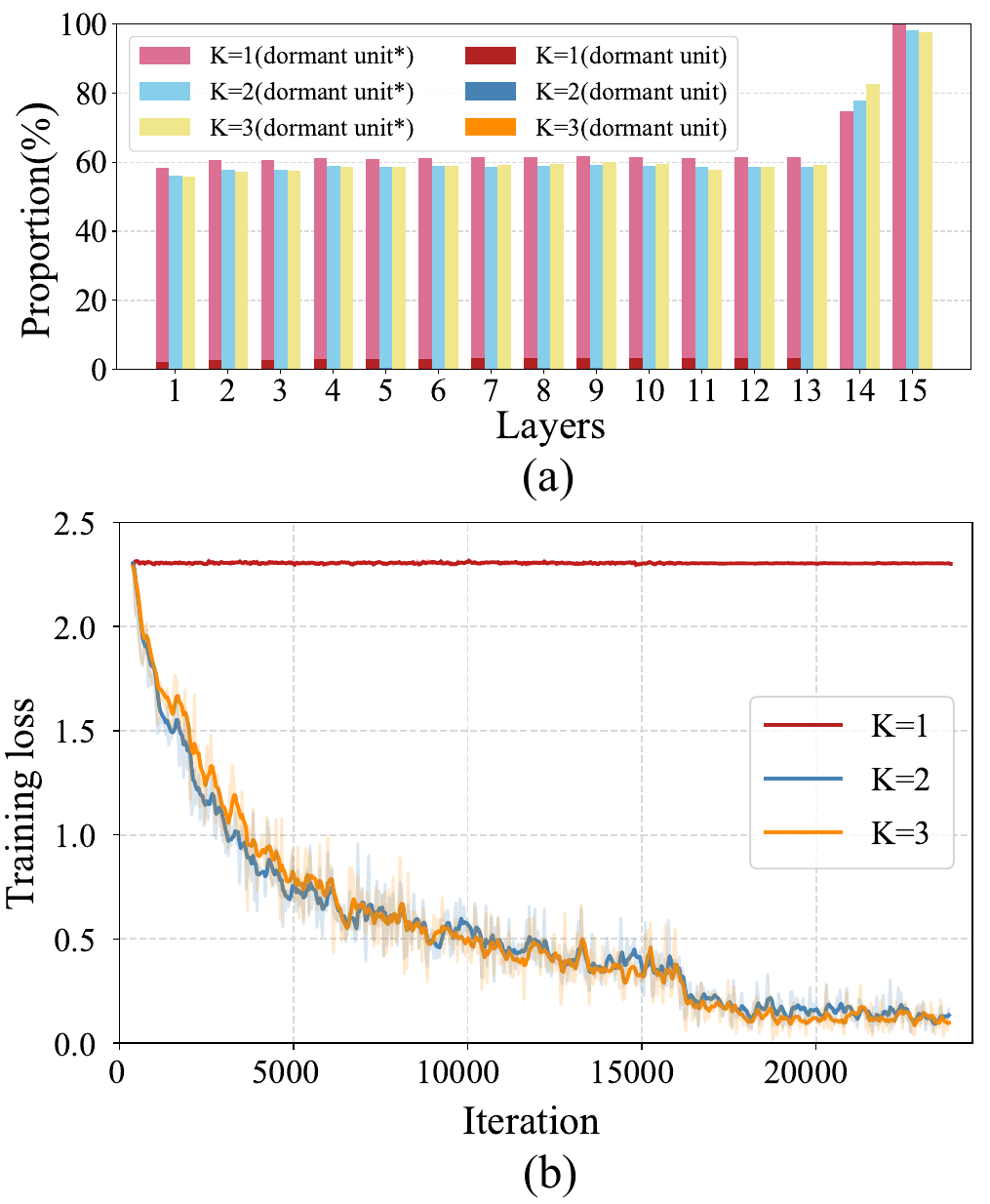} 
\caption{(a) The average proportion of dormant units in each layer of VGG16. (b) The training loss of VGG16. Dormant unit* represents the dormant unit located in the saturation area on the left side of the approximate derivative $h_1(u_1)$.}
\label{fig_vgg}
\end{figure}

\subsection{In the Middle of Training}
When SNNs go deeper, existing direct training methods will not only suffer from network degradation but also the catastrophic gradient vanishing. In the middle of training, the sharp features cannot be further improved after exceeding the firing threshold due to the limited width of $h_1(u_1)$, resulting in firing instability. One layer may not fire any spikes after training a mini-batch, the distribution of which is quite different from the overall distribution. In this case, there will be no gradient for backpropagation, resulting in training crashing.

To demonstrate this catastrophic gradient vanishing, we take ResNet (32-layer, middle) as basic architecture, and conduct three groups of experiments on CIFAR10. The reported deepest network structure (ResNet-SNN) is taken as the baseline. The model of each group is trained with five different random seeds (seed=0, 1, 2, 3, 4). The level of MLF is set to 2.

\begin{figure*}[h]
\centering
\includegraphics[width=1\columnwidth]{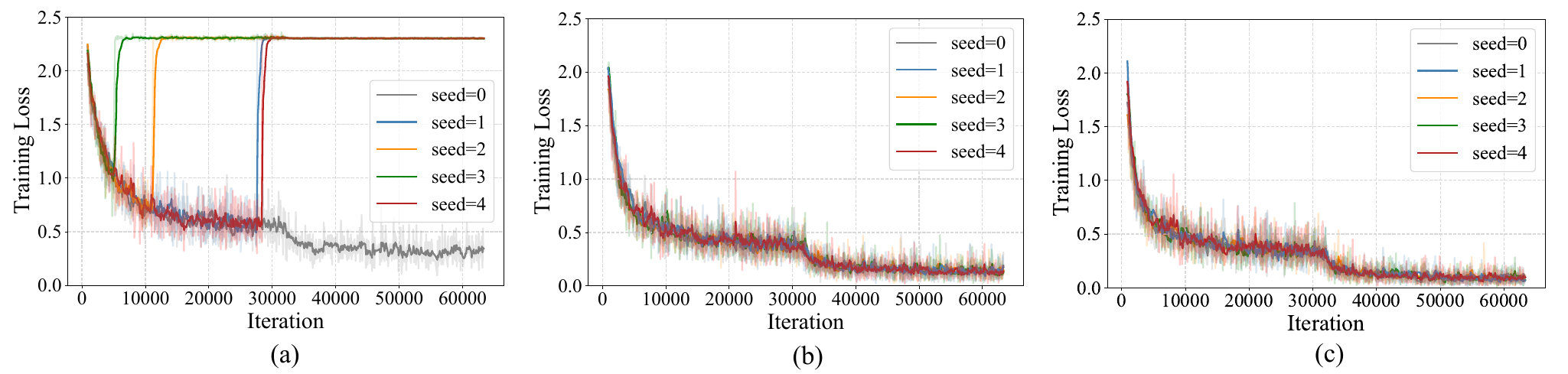} 
\caption{Training loss during the whole training process with five different random seeds. (a) ResNet-SNN without MLF. (b) ResNet-SNN with MLF ($K=2$). (c) Spiking DS-ResNet without MLF.}
\label{fig_stability}
\end{figure*}
As shown in Fig.~\ref{fig_stability}(a), four-fifths of the ResNet-SNN training crash during the training process. When $K=1$, the sharp feature cannot be further improved after exceeding the firing threshold due to the limited width of $h_1(u_1)$, resulting in firing instability. It is likely that after training a mini-batch data, the distribution of which is quite different from the overall distribution, a layer may not fire any spikes. This instability of training will increase when SNNs go deeper. Therefore, in order to avoid training crashing caused by catastrophic gradient vanishing, the hyper-parameter setting and the initialization of network parameters will need to be more elaborate. Our model (whether MLF or spiking DS-ResNet) can ensure the stability of the whole training of very deep SNNs without elaborate hyper-parameter setting and initialization of network parameters, as shown in Fig.~\ref{fig_stability}(b), (c).

\section{Additional Cost of MLF Compared with LIF}
We analyze the cost of MLF from both software and hardware perspective. The FLOPs of LIF and MLF in one layer are as follows
\begin{align}
F_1&=2T(\underbrace{N_{k}^{2}*C_{in}*M^{2}*C_{out}}_{\rm Convolution}+\underbrace{M^{2}*C_{out}}_{\rm LIF}), \\
F_2&=2T(\underbrace{N_{k}^{2}*C_{in}*M^{2}*C_{out}}_{\rm Convolution}+\underbrace{M^{2}*C_{out}*K}_{\rm MLF}),
\end{align}
where $F_1$ and $F_2$ denote the FLOPs of LIF and MLF respectively. $N_{k}$ and $M$ are the size of the kernel and feature map respectively. $C_{in}$ and $C_{out}$ are the number of input and output channels respectively. As we can see, the additional cost of MLF is small, and the complexity of LIF and MLF are both $o(T*K^{2}*C_{in}*M^{2}*C_{out})$.

From software and hardware perspective, we estimate the FLOPs and the spike number of spiking ResNet (20-layer, middle) on CIFAR10 with $T=4$ and $K=3$ during a single inference process. The results show that the FLOPs of LIF and MLF are about $1.3697\times10^{9}$ and $1.3757 \times10^{9}$ respectively, and the spike number of LIF and MLF are about $3.87\times10^{5}$ and $6.75 \times10^{5}$ respectively (the energy consumed to transmit a spike on neuromorphic hardware is only nJ or pJ~\cite{DiehlCook2015}). In summary, compared with LIF, MLF can be applied with little additional cost from software perspective. From hardware perspective, MLF will fire more spikes for communication between neurons.

\end{document}